\documentclass{article}

\newif\ificml
\newif\ifcolm
\colmtrue

\usepackage{microtype}
\usepackage{graphicx}
\usepackage{subfigure}
\usepackage{multirow}
\usepackage{enumitem}
\usepackage{wrapfig}
\usepackage{booktabs} 

\usepackage{siunitx}
\usepackage{hyperref}

\usepackage{amsmath}
\usepackage{amssymb}
\usepackage{mathtools}
\usepackage{amsthm}
\usepackage{thmtools}
\usepackage{bbm}

\usepackage[capitalize,noabbrev]{cleveref}

\theoremstyle{plain}
\newtheorem{theorem}{Theorem}[section]

\theoremstyle{definition}

\theoremstyle{remark}
\newtheorem{remark}[theorem]{Remark}
\usepackage{xspace}

\newcommand{\Update}{\texttt{update}}
\newcommand{\UCB}{\operatorname{UCB}}
\newcommand{\EUCB}{\operatorname{EUCB}} 

\newcommand{\LCB}{\operatorname{LCB}}
\newcommand{\Gen}{\operatorname{Gen}}
\theoremstyle{plain}

\newcommand{\configSpace}{\mathcal{X}}
\newcommand{\configSet}{\mathcal{C}}

\newcommand{\tool}{\texttt{QuickScope}\xspace}

\newtheorem{property}{Property}[section]
\theoremstyle{remark}

\newif\ifcomments
\commentstrue
\ifcomments
    \providecommand{\narun}[1]{{\protect\color{blue}{[Narun: #1]}}}
    \providecommand{\tl}[1]{{\protect\color{magenta}{[Taylor: #1]}}}  
    \providecommand{\klb}[1]{{\protect\color{red}{[Kevin: #1]}}} 
\else
    \providecommand{\narun}[1]{}
    \providecommand{\tl}[1]{} 
    \providecommand{\klb}[1]{} 
\fi

\usepackage[textsize=tiny]{todonotes}

\usepackage{multiaudience}
\SetNewAudience{colm}
\SetNewAudience{icml}

\ificml
    
    \DefCurrentAudience{icml}
    
    \usepackage{icml2025}
    
    \icmltitlerunning{\texttt{QuickScope}: Certifying Hard Questions in Dynamic LLM Benchmarks}
\else\ifcolm
    \usepackage[preprint]{colm2026_conference}
    
    \DefCurrentAudience{colm}
    \usepackage{microtype}
    
    \usepackage{lineno}
    
    \definecolor{darkblue}{rgb}{0, 0, 0.5}
    \hypersetup{colorlinks=true, citecolor=darkblue, linkcolor=darkblue, urlcolor=darkblue}

    \title{QuickScope: Certifying Hard Questions in Dynamic LLM Benchmarks}
    
    \author{Taylor Lundy, Narun K.~Raman\thanks{ Corresponding author}, \ \& Kevin Leyton-Brown  \\
    Department of Computer Science\\
    University of British Columbia\\
    Vancouver, BC \\
    \texttt{\{tlundy,narunram,kevinlb\}@cs.ubc.ca} 
    }
    
\fi

\begin{document}

\begin{shownto}{icml}
    
    \twocolumn[
    \icmltitle{\texttt{QuickScope}: Certifying Hard Questions in Dynamic LLM Benchmarks}
    
    
    
    \icmlsetsymbol{equal}{*}
    
    \begin{icmlauthorlist}
    \icmlauthor{Firstname1 Lastname1}{equal,yyy}
    \icmlauthor{Firstname2 Lastname2}{equal,yyy,comp}
    \icmlauthor{Firstname3 Lastname3}{comp}
    \icmlauthor{Firstname4 Lastname4}{sch}
    \icmlauthor{Firstname5 Lastname5}{yyy}
    \icmlauthor{Firstname6 Lastname6}{sch,yyy,comp}
    \icmlauthor{Firstname7 Lastname7}{comp}
    \icmlauthor{Firstname8 Lastname8}{sch}
    \icmlauthor{Firstname8 Lastname8}{yyy,comp}
    \end{icmlauthorlist}
    
    \icmlaffiliation{yyy}{Department of XXX, University of YYY, Location, Country}
    \icmlaffiliation{comp}{Company Name, Location, Country}
    \icmlaffiliation{sch}{School of ZZZ, Institute of WWW, Location, Country}
    
    \icmlcorrespondingauthor{Firstname1 Lastname1}{first1.last1@xxx.edu}
    \icmlcorrespondingauthor{Firstname2 Lastname2}{first2.last2@www.uk}
    
    
    \vskip 0.3in
    ]
    
    
    
\end{shownto}    

\ifcolmsubmission
\linenumbers
\fi
    
\begin{shownto}{colm}
    \maketitle
\end{shownto}
    
\begin{abstract}
LLM benchmarks are increasingly {dynamic}: instead of containing a fixed set of questions, they define templates and parameters that can generate an effectively unlimited number of question variants. This flexibility is valuable, but it makes evaluation expensive—especially when the goal is not just determining an average score, but reliably identifying a model's {weak spots}. This paper introduces a new methodology for identifying hard questions in dynamic benchmarks. It leverages COUP, a recent Bayesian optimization algorithm (Graham, Velez \& Leyton-Brown, 2026), after introducing several substantive modifications to make the algorithm suitable for practical LLM pipelines. We also wrap it in a tool that supports flexible choices of datasets and utility functions, enabling users to target the kinds of questions they care about (e.g., low-accuracy questions; questions that are unusually hard relative to their measured complexity). In experiments across a range of benchmarks, we show that our method, dubbed \emph{QuickScope}, discovers truly difficult questions more sample efficiently than standard baselines, while also reducing false positives from noisy outcomes.
\end{abstract}

\section{Introduction}\label{sec:intro} 
Benchmarking is one of the main ways we track progress in large language models (LLMs). But as model performance rises, two practical issues have become increasingly hard to ignore. First, static test sets are growing increasingly vulnerable to overfitting: repeated optimization against a fixed benchmark can encourage improvements that reflect adaptation to the benchmark rather than underlying capability.
Second, at the frontier, models are strong enough that overall accuracy differences can be small, and what we need from benchmarks is often diagnostic: to support downstream decisions, it matters which problems a model misses and whether those failures persist when we restate the question in slightly different words or resample the names and numbers without changing the underlying problem. However, such fine-grained diagnostics are difficult to obtain from static benchmarks, which typically contain only a small number of questions per concept and often lack a clear notion of similarity between questions.


In response, there is a growing tendency to make benchmarks {dynamic}. Instead of a fixed list of questions, a dynamic benchmark specifies {parameters} that can generate an effectively unbounded set of questions. For example, a question might be parameterized by the underlying concept it evaluates (e.g., addition), its domain (e.g., a medical setting) and its perspective (e.g., first person) and then the question can be instantiated using randomly generated values. This makes it possible for many i.i.d.\ questions to be instantiated, improving the statistical power with which performance on a concept can be measured.

Dynamic benchmarks most significantly eclipse static alternatives when the goal is to identify a model's {weak spots}, as compared e.g. to its average performance under some sampling distribution. In such cases, it is important to surface questions that are \emph{truly} difficult for a target model, rather than simply highlighting occasional failures due to stochasticity, formatting quirks, or other transient effects. Put differently, the goal is not merely to find failures, but to find questions or regions that can be {certified} as hard with high confidence. Such certification requires prospective weak regions to be sampled densely. It is thus necessary either to employ uniform random sampling and to sample easy regions just as densely, or to employ some dynamic sampling strategy. 

This paper takes the dynamic approach, framing evaluation on dynamic benchmarks as a sequential decision problem. Each query incurs cost; outcomes are noisy; and the space of possible questions is structured by the benchmark's 
parameterization. The evaluator's job is to allocate budget adaptively to match an explicit user objective. The objective might be ``find low-accuracy questions,'' but it can also be richer, such as ``find questions that are unusually hard relative to their measured complexity'' or ``find small changes in the question that cause a big drop in accuracy.'' 
Our approach brings Bayesian optimization to benchmark evaluation by building on COUP \citep{graham2025practicalutilitarianalgorithmconfiguration, graham2025utilitarianalgorithmconfigurationinfinite}, a recent UCB algorithm originally developed for cost-aware algorithm configuration. We treat each benchmark template as a configuration and use the user-defined objective to define payoffs; COUP then decides which parts of the benchmark to query next by sampling based on uncertainty-aware upper confidence bounds. However, using COUP off-the-shelf is not sufficient for modern LLM pipelines. LLM evaluation is naturally batched and parallel, and many preferences are over {aggregate performance} (e.g., long-run accuracy) rather than per-sample performance. We therefore modify COUP to (i)~operate effectively under parallel evaluation and (ii)~support utility mappings that let users express preferences over long-run aggregates instead of single outcomes.  

We package these ideas into a practical tool for dynamic benchmarks. We call this tool \tool---it quickly scopes out a model's weak spots within a benchmark's configuration space---and it cleanly separates (a) the benchmark interface that maps a configuration to a concrete question instance, (b) the model interface that executes the evaluation, and (c) a utility function that captures the user's notion of ``interesting'' or ``hard.'' This separation is intentional: dynamic benchmarks vary widely in how their parameters are structured and how questions are instantiated, and different evaluation goals call for different utility definitions. By making these components explicit and swappable, the same underlying algorithm can target low-accuracy regions, prioritize questions that are hard conditional on complexity, or implement other user-specified objectives without requiring any changes to the core optimization logic. 
Empirically, across benchmarks, we show that our method discovers truly difficult questions with more sample efficiency than standard baselines, while reducing false positives that arise from noisy outcomes. In settings where users care about reliable identification of weak spots, this translates into fewer total model calls to obtain a comparable (or better) set of certified-hard items.  

The  remainder of this paper proceeds as follows. \Cref{sec:related} reviews prior work on LLM evaluation, dynamic benchmarks, and algorithm configuration methods; it also 
provides the COUP background needed for our setting, including its anytime guarantees and the roles of its confidence bounds and proposal mechanism. With this setup in place, \Cref{sec:coup} introduces our extensions that make COUP applicable to practical LLM pipelines---in particular, batched/parallel evaluation and utility functions defined over long-run aggregates rather than single noisy outcomes. \Cref{sec:tool} describes our tool, \tool, and its extension points. \Cref{sec:results} evaluates our approach against standard baselines across benchmarks, and \Cref{sec:discussion} discusses limitations and broader implications for evaluation design.

\section{Related Work and Background}\label{sec:related}
\paragraph{LLM Evaluation}
A large literature studies how to make LLM evaluation more informative than a single headline score, including broader metric suites, robustness checks, and controlled prompting protocols \citep[e.g.,][]{liang2022helm,chang2024benchmarking_survey,openai2023evals,amodei2016concreteproblemsaisafety,bommasani2022opportunitiesrisksfoundationmodels, srivastava2023imitationgamequantifyingextrapolating}. 
Another growing body of work highlights the fragility of static benchmarks under repeated iteration, training--test overlap, and shifting evaluation conditions  \citep[e.g.,][]{lee-etal-2022-deduplicating, singh2024evaluationdatacontaminationllms, fang2025lastingbenchdefendbenchmarksknowledge, carlini2023quantifyingmemorizationneurallanguage, carlini_2021_extracting,dror-etal-2018-hitchhikers, bouthillier2021accountingvariancemachinelearning, raman2025reasoningmodelstestexploiters}. These concerns motivate \emph{dynamic} benchmarks, where a benchmark specifies generators or templates that can produce many question variants rather than a fixed test set.

Recent work has begun to explore how to search such spaces efficiently. Most of these benchmarks generate questions from formal rules, simulators, or solvers \citep[e.g.,][]{kiela2021dynabenchrethinkingbenchmarkingnlp, majdinasab2025prismdynamicflexiblebenchmarking, wang-etal-2025-benchmark, raman2025steermeassessingmicroeconomicreasoning, mirzadeh2025gsmsymbolicunderstandinglimitationsmathematical, zhu2024dyvaldynamicevaluationlarge} and cover that space uniformly, while some recent work has begun to use surrogate models or adaptive sampling to guide evaluation \citep{zhu2024dyvaldynamicevaluationlarge, majdinasab2025prismdynamicflexiblebenchmarking}. However, these methods are often tailored to their specific generators and do not directly cover the broader class of benchmarks we target. Others consider the discovery of challenging examples \citep[e.g.,][]{kiela2021dynabench, perez2022discoveringlanguagemodelbehaviors, hines2024defendingindirectpromptinjection, belaire2025automatic, zou2023universaltransferableadversarialattacks}, rather than certifying that a region is \emph{reliably} hard under repeated random instantiation. Our goal is precisely this latter problem: to identify configurations that can be certified as hard with high confidence.

\paragraph{Algorithm Configuration and Hyperparameter Optimization}
Our approach builds most directly on work in algorithm configuration and hyperparameter optimization, which likewise treats evaluation as expensive and stochastic. Model-based approaches such as SMAC can achieve strong practical sample efficiency, but do not offer theoretical guarantees or certificates \citep{hutter2011smbo,lindauer2019assessingimpactbayesianoptimizations,lindauer2022smac3versatilebayesianoptimization}. Bandit-style methods provide more theoretically grounded treatment of uncertainty. 
Notably, Hyperband \citep{li2018hyperbandnovelbanditbasedapproach} and its successor algorithms  Asynchronous Successive Halving \citep[ASHA][]{li2020massivelyparallelhyperparametertuning}, and Bayesian Optimization Hyperband \citep[BOHB][]{falkner2018bohbrobustefficienthyperparameter} achieve good wall-clock efficiency, but run to a prespecified budget rather than providing anytime certificates of solution quality. More recently, COUP \citep{graham2025practicalutilitarianalgorithmconfiguration, graham2025utilitarianalgorithmconfigurationinfinite} offers similar guarantees in an anytime algorithm. It is furthermore designed for \emph{certified} identification in large (even infinite) spaces under user-defined utilities and can leverage SMAC-like proposal models with little impact on worst-case performance bounds.

\paragraph{COUP}
Because we build upon it, we describe the COUP algorithm in some detail. COUP optimizes over a configuration space $\configSpace$. In the original COUP setting, a configuration corresponds to a choice of algorithm and its hyperparameters, such as for a combinatorial optimization algorithm. 
COUP's goal is to identify a configuration with high expected utility and to provide statistical guarantees about its quality. A naive optimistic strategy for finite $\configSpace$ is to repeatedly sample the configuration with the largest upper confidence bound (UCB). This is good for minimizing regret, but we're often more interested in simply identifying the highest-utility configuration (so-called best-arm identification). An alternative algorithm called LUCB \citep{kalyanakrishnan2012pac} is better in this case. It samples two configurations each round: the one with the highest empirical mean and, separately, the configuration with the largest UCB among the remaining options. This balances exploration and exploitation, helping to more quickly gather information about the most promising configurations.

When $\configSpace$ is very large, neither approach is practical: untested configurations initially have vacuous (large) confidence bounds, which drives excessive exploration and makes it difficult to concentrate evaluations on a small set of promising candidates. COUP addresses this by maintaining a small set of configurations $\configSet \subseteq \configSpace$ and alternating between: (i) allocating evaluations within this set (e.g., sampling the empirical-best and a UCB-based challenger within $\configSet$), and (ii) expanding the set of configurations by proposing a new configuration from $\configSpace$. A fraction (say, half) of these new configurations is drawn uniformly at random to ensure the configuration-space coverage required for COUP's guarantees; the remainder may be proposed by any predictive model trained on the utilities observed so far.
Often, configuration spaces possess an underlying structure that allows for extrapolation from evaluated configurations to similar, unseen ones. COUP leverages this structure both to accelerate discovery and to maintain diversity. 

COUP balances a second exploration-exploitation problem---learning more about its current set of configurations vs.\ expanding this set---using two internal control quantities, $\epsilon$ and $\gamma$. Informally, $\epsilon$ captures how well COUP has separated the best configuration from the rest: it is the gap between the lower confidence bound (LCB) of the configuration with the highest LCB (the ``incumbent'') and the largest UCB of any configuration.
Thus, up to the chosen failure probability, the incumbent has expected utility within $\epsilon$ of the true best configuration in the set. The quantity $\gamma$ controls how much of the broader configuration space COUP can confidently \emph{rule out} as beating the incumbent: larger $\gamma$ corresponds to certifying that a larger-measure subset of $\mathcal{X}$ (under COUP's proposal distribution) is unlikely to contain a configuration with higher utility than the incumbent. COUP expands $\configSet$ when the potential benefit of searching the larger space outweighs that of further sampling configurations in the current set; concretely, it compares $\epsilon$ and $\gamma$ via a condition of the form $\epsilon^2 \le c \gamma$ for a user-controlled constant $c$.  COUP's formal guarantee is stated in terms of $(\epsilon,\gamma)$-optimality. In the infinite-space setting, COUP defines $\mathrm{OPT}_\gamma$ as the top-$\gamma$ quantile of expected utilities under the configuration distribution, and calls a configuration $a$ $(\epsilon,\gamma)$-optimal if $U_a \ge \mathrm{OPT}_\gamma - \epsilon$. COUP guarantees (with probability at least $1-\delta$) that at the end of each timestep $p$ it returns an $(\epsilon_p,\gamma_p)$-optimal configuration; see Definition~2 and Theorem~2 in \citet{graham2025utilitarianalgorithmconfigurationinfinite}.

\section{Extending COUP for Benchmark Evaluation} \label{sec:coup}

This section describes the extensions we make to COUP to fit practical LLM benchmark evaluation. We introduce \textit{certification} to efficiently identify hard configurations, \textit{repulsion} to surface diverse high-utility configurations, and a \textit{batching rule} that enables parallel evaluation while preserving conservative bound behavior.

COUP optimizes over a configuration space $\mathcal{X}$. In our application, we take $\mathcal{X}$ to be the benchmark’s space of \emph{template identifiers}. A template identifier $x \in \mathcal{X}$ is the minimal discrete descriptor that selects a particular question-generating procedure (and any benchmark-defined structure we might want to condition on, such as a concept, domain, or reasoning primitive). Depending on the benchmark (see, e.g., \Cref{app:benchmarks}), $x$ may be a structured tuple of attributes or simply an index into a finite list of templates.

Evaluating a template identifier $x$ consists of two steps. First, the benchmark \emph{instantiates} the template by sampling a concrete question instance (e.g., by resampling names, numbers, or other slots that preserve the underlying problem). Second, we run the target model on this instantiated question to obtain an outcome and/or metrics (e.g., correctness, log-probability, latency, cost), which we map to a scalar utility via a user-specified function $u$. Repeating this process yields i.i.d.\ utility draws for each $x$, so COUP can allocate evaluations across templates and maintain confidence bounds on each template's mean utility.

\paragraph{Certification and budget reallocation.}
As described above, COUP maintains confidence bounds $[\LCB_i(S), \UCB_i(S)]$ on each configuration's mean utility $\theta_i$ as samples accumulate. In our setting, $\theta_i$ is the expected value of the user-specified sample utility (e.g., inverse accuracy or complexity-weighted error) for template~$i$. These bounds let COUP decide where to allocate samples, but in practice a user searching for hard questions may be indifferent to fine-grained distinctions among configurations that are all clearly above some hardness threshold. For example, on a 4-choice MCQA benchmark, any template whose error rate exceeds $0.75$ (the chance level) is unambiguously hard; continuing to distinguish between $0.80$ and $0.90$ error rate wastes budget that could be spent exploring other templates. Without intervention, COUP continues to allocate samples to its top configurations indefinitely, refining their confidence bounds. We introduce a \emph{certification} mechanism that removes a configuration from the active set once its lower confidence bound exceeds a user-specified threshold~$\tau$:
 if $\LCB_i(S) \ge \tau$
then configuration $i$ is certified and removed.
Once certified, a configuration no longer competes for samples, and the freed budget is redirected toward configurations with higher residual uncertainty. The threshold $\tau$ is a tunable parameter that controls the precision/coverage trade-off: a higher $\tau$ demands stronger evidence before certifying (tighter bounds at the top, fewer configurations certified), while a lower $\tau$ certifies earlier and spreads the budget more broadly. 

We also consider a softer certification rule, which we call \emph{adaptive certification}, that lies between the threshold-based rule above and continuing to refine top configurations without certification. As before, a configuration first enters the certified set when its lower confidence bound exceeds a user-specified threshold $\tau$. However, certified configurations remain active rather than being removed, and if $K$ configurations are currently certified, the algorithm treats the problem as maintaining the top $K$ hard configurations. It then allocates samples according to the LUCB-$K$ algorithm: the current top-$K$ set is defined as the $K$ configurations with highest empirical mean utility, and sampling alternates between the member of this set with the smallest lower confidence bound and the configuration outside the set with the largest upper confidence bound. As the certified set grows, the certification threshold is also raised. The new threshold is chosen to balance the evidence required to admit one more configuration against the evidence required to bring all $K$ configurations already in the set up to this stricter standard. The result is an intermediate procedure that is useful when the user has only a rough sense of the right hardness threshold: $\tau$ still anchors what counts as plausibly hard, but the algorithm adaptively balances the size of the certified set against the confidence with which that set is identified.


\paragraph{Diversity via repulsion among near-optimal configurations.}
Users may prefer to surface a \emph{diverse} set of high-utility configurations rather than a tight cluster of near-duplicates. We support this by augmenting the objective with a repulsion penalty that downweights candidates that are semantically or structurally close to a user-specified reference set. This reference set can be defined for any tolerance $\varepsilon$ (e.g., all configurations whose LCB is within $\varepsilon$ of the maximum observed utility). In our experiments we use the simplest instantiation, $\varepsilon=0$, so the reference set is exactly the set of \emph{certified} configurations.

To make repulsion work on structured identifier spaces without a hand-designed distance metric, we derive a simple proximity metric  from the random-forest surrogate itself. For two identifiers, we compute their similarity as the fraction of trees in which they land in the same leaf, or equivalently, the fraction of such shared leaf nodes across the ensemble. This shared-leaf metric automatically weights different identifier attributes by their learned importance (since splits on more informative attributes occur more often and earlier), yielding an interpretable, plug-and-play notion of ``closeness'' for repulsion.

\paragraph{Batching}
A practical issue in our setting is that COUP is inherently sequential: at each step it identifies (i) the configuration with the highest empirical mean and (ii) the configuration with the highest UCB, and its next decision depends on the realized outcomes of these evaluations. Because the procedure is deterministic, we cannot simply rerun the same step to obtain additional configurations for parallel evaluation. This is a problem because practical LLM evaluation requires highly parallel computation. We therefore need a proxy for the configurations COUP would have selected in sequence, so that we can bundle them into a batch and evaluate them in parallel.

To formalize what we want from batching, we compare against a simple reference process. Starting from the real pre-batch state $S_0$, define the \emph{expected-outcome world} as the sequential run of COUP in which each evaluation returns its conditional expectation (i.e., the configuration's mean reward), rather than a random draw.

We introduce a batching method that constructs an $N$-sized batch using only pre-batch information while aiming for two guarantees: (G1: \emph{conservativeness}) with high probability, the temporary (within-batch) confidence bounds used to construct the batch are no more optimistic than the bounds that would arise under the corresponding (counterfactual) real sequential updates; and (G2: \emph{coverage}) the batch contains a superset of the configurations COUP would evaluate over the same horizon in the expected-outcome world.

We defer the complete description and proofs of the batching procedure, which we call Frozen-LCB simulation, to the appendix (\Cref{app:proof}). At a high level, this procedure simulates $N$ steps of the COUP algorithm as if each returned sample comes from that configurations LCB. We first prove that this procedure satisfies (G1).

\begin{restatable}[Frozen-LCB simulated means are $(1-\delta)$-conservative for expected UCB]{theorem}{firsttheorem}
\label{thm:coup_fantasy}
Consider any batching procedure that performs $K$ simulated steps. At each step $t=1,\dots,K$, the procedure selects some configuration $i_t\in\mathcal{I}$ (possibly adaptively, using $D$ and all previous simulated steps), with current internal state $S_t$, and simulates sampling the frozen mean $\widetilde{m}_t \coloneqq L_{i_t}$.
Under Assumptions~\ref{prop:coup_monotone}--\ref{prop:coup_lcb_valid}, with probability at least $1-\delta$, the following holds \emph{simultaneously for all} $t=1,\dots,K$:
\(
\EUCB_{S_t,i_t}(L_{i_t}) ~\le~ \EUCB_{S_t,i_t}(\mu_{i_t}).
\)
Equivalently, at every simulated step, using $L_{i_t}$ as the simulated mean yields an expected post-update UCB value that is no larger than the expected post-update UCB value under the true mean, except with probability at most $\delta$.
\end{restatable}

We then show that Frozen-LCB simulation also satisfies (G2), establishing that the resulting batch covers the configurations that COUP would select over the same horizon in the expected-outcome world.

\begin{restatable}{theorem}{secondtheorem}
\label{thm:sim_dominance}
Let $\mathcal{S}_{\text{real}}$ and $\mathcal{S}_{\text{sim}}$ be the sets of unique configurations sampled by the true sequential process (in expectation) and the frozen-LCB simulation, respectively, given a fixed budget $K$.
Under Assumptions \ref{prop:coup_monotone}--\ref{prop:local_updates}, with probability at least $1-\delta$:
 The simulation explores a superset of the configurations explored by the true process:
    $ \mathcal{S}_{\text{real}} \subseteq \mathcal{S}_{\text{sim}} $
    and if the containment is strict ($\mathcal{S}_{\text{real}} \subset \mathcal{S}_{\text{sim}}$), then the average number of samples allocated to the configurations in $\mathcal{S}_{\text{real}}$ is strictly lower in the simulation than in the true process.
\end{restatable}

\section{\tool flexibility and extensibility} \label{sec:tool}
\tool separates the COUP-based optimizer from benchmark-specific design choices. To apply it to a new benchmark, users must specify three components: a utility, a search space over template identifiers, and a generator. They can optionally also specify a predictive model, in the case that the generator class has enough structure to support extrapolation to nearby unseen configurations. 

\paragraph{Utility and Search Space Specification}
COUP optimizes a scalar utility in $[0,1]$, which lets users encode different evaluation goals without changing the optimizer. A utility function maps an evaluation outcome to a scalar:
$
u: (\{\text{metrics}\}, x) \mapsto [0,1],
$
where $x\in\mathcal{X}$ is the template identifier and $\{\text{metrics}\}$ are the scored outputs which can include any diagnostics of interest. A standard choice is error rate, $u(\{\text{metrics}\},x)=1-\texttt{acc}(x)$. More structured utilities can incorporate template properties such as complexity $c(x)$ (e.g., the size of a graph when solving shortest path) to scale or normalize the error rate. COUP then just maintains confidence bounds on each template's mean utility $\theta_x \coloneqq \mathbb{E}[u(\{\text{metrics}\},x)]$ as samples accrue. The search space specifies which parameters constitute the identifier and what values are admissible. Importantly, users are not restricted to the default space: they can also narrow or reshape an existing dataset's search space (e.g., focus on a subset of domains or isolate a parameter subregion) while reusing the same generator and evaluation code.

\paragraph{Generator}
The only required dataset-specific logic is a generator
$
\Gen: x \mapsto \text{instance}(x;\omega),
$
which maps a template identifier $x$ to a concrete question by sampling within-template randomness $\omega$ (e.g., numeric values, names, contexts, or seeds). This design directly supports dynamic benchmarks: the search space identifies what kind of question we want, and the generator produces a fresh instantiation of that kind.

\paragraph{Predictive Model} \label{subsec:predmodel}
\tool uses two predictive components: a utility model $f:\mathcal{X}\to[0,1]$ for expansion and a proximity metric $d(x,x')$ for repulsion. These can be specified independently---for example, a Gaussian process for utility prediction and a pretrained embedding model for distance---but we use a Random Forest by default, which provides a good balance between predictive performance and retraining cost.

\section{Experimental Evaluation}\label{sec:exp_setup}\label{sec:results}

We evaluated \tool on three dynamic benchmarks---DyVal, Grid Reasoning, and STEER-ME---using error-rate utility. We additionally evaluated complexity-weighted error (CWE) on DyVal, defined as $(1 - \mathrm{acc}) \,/\, \log_2(\mathrm{pow}(n_{\mathrm{children}},{depth}))$, which upweights errors on shallower, simpler DAGs. Benchmark details are deferred to \Cref{app:benchmarks};  COUP hyperparameters, scalability modifications, prompting, and inference settings are in \Cref{app:experimental_setup}. Each run used a budget of \num{20000} model calls. We compared \tool to a uniform-sampling baseline with the same total budget. For DyVal and Grid Reasoning, we compared certification at the \qty{90}{\percent} level to no certification; for STEER-ME, whose 4-choice MCQA format has chance error rate \qty{75}{\percent}, we used a \qty{70}{\percent} certification threshold. We then re-evaluated each method's top configurations so that each had at least \num{200} total samples. For space considerations, we focus on \texttt{gpt-5.4-nano-2026-03-17} in the main text; we also ran a slightly older version of \tool on \texttt{gpt-5-mini}, wherein we observed the same trends, and is deferred to \Cref{app:cross_model}. 

\begin{shownto}{icml}
    \begin{figure*}[t]
        \centering
        \includegraphics[width=\linewidth]{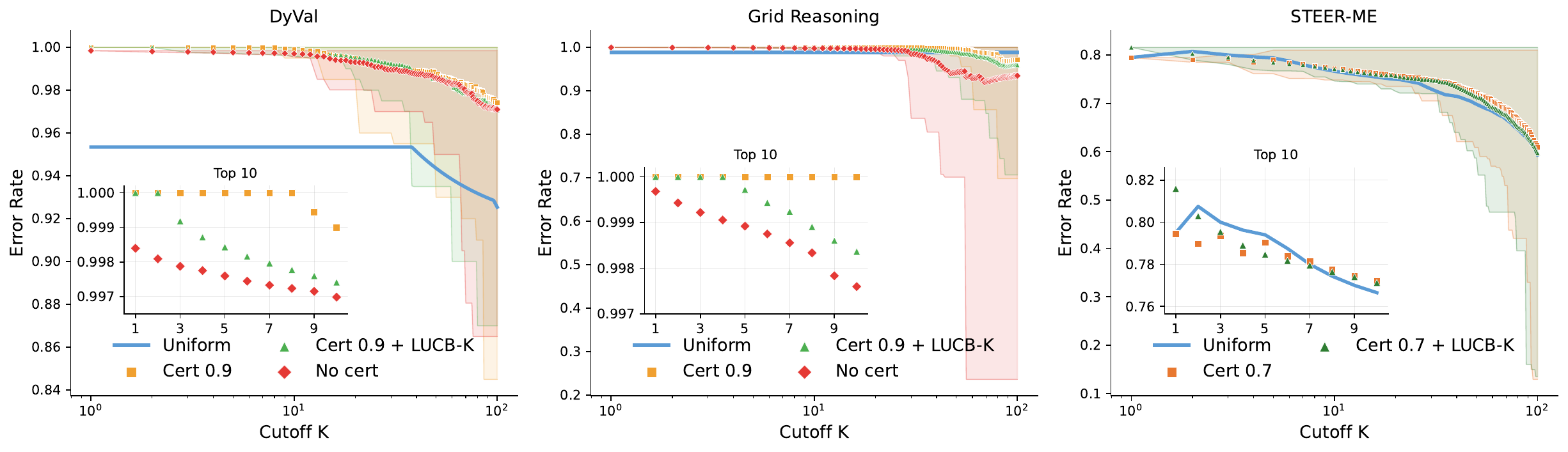}
        \caption{Cumulative average re-evaluation utility as a function of cutoff~$K$ (log-scaled $x$-axis). COUP variants (markers) are ranked by LCB; shaded bands show the running min/max re-evaluated utility, so non-monotonicities reflect individual configurations whose re-evaluated utility falls below the running average. Uniform sampling (solid line) is ranked by empirical mean; because uniform allocation produces many tied means, the line shows the expected cumulative average over random tie-breaking within each tie group. Inset panels zoom into the top~10 configurations.}
        \label{fig:cumavg_reeval}
    \end{figure*}
\end{shownto}


To validate that \tool's top-ranked configurations were truly difficult, we ranked configurations by their selection criterion---lower confidence bound (LCB) for COUP variants, empirical mean for uniform sampling---and plotted the cumulative average re-evaluation utility across the top-$K$ configurations as a function of~$K$ (\Cref{fig:cumavg_reeval}). For each method, shaded bands show the running minimum re-evaluated utility across the top~$K$. Because uniform sampling produces many configurations with the same empirical mean, their ranking within tied groups is arbitrary; the uniform markers therefore show the expected cumulative average over random tie-breaking. Thus, high re-evaluated utility at small~$K$ confirms that COUP is identifying truly hard configurations rather than artifacts of adaptive sampling. 

\Cref{app:weak_spots} gives additional detail on the top configurations found by each variant. In brief, DyVal's top-5 were all deep, narrow arithmetic tasks, consistent with prior observations that depth and compositional structure are key drivers of difficulty in DyVal-style benchmarks. However, whereas prior analyses identify challenging regimes at the level of task families or complexity trends, our method localizes this difficulty to specific regions of the configuration space. On Grid Reasoning, the hardest configurations were larger grids, with no~cert and adaptive~cert agreeing on the top two; and on STEER-ME, numerical \emph{intertemporal consumption smoothing} dominated.

\begin{shownto}{colm}
    \begin{figure}[t]
        \centering
        \includegraphics[width=\linewidth]{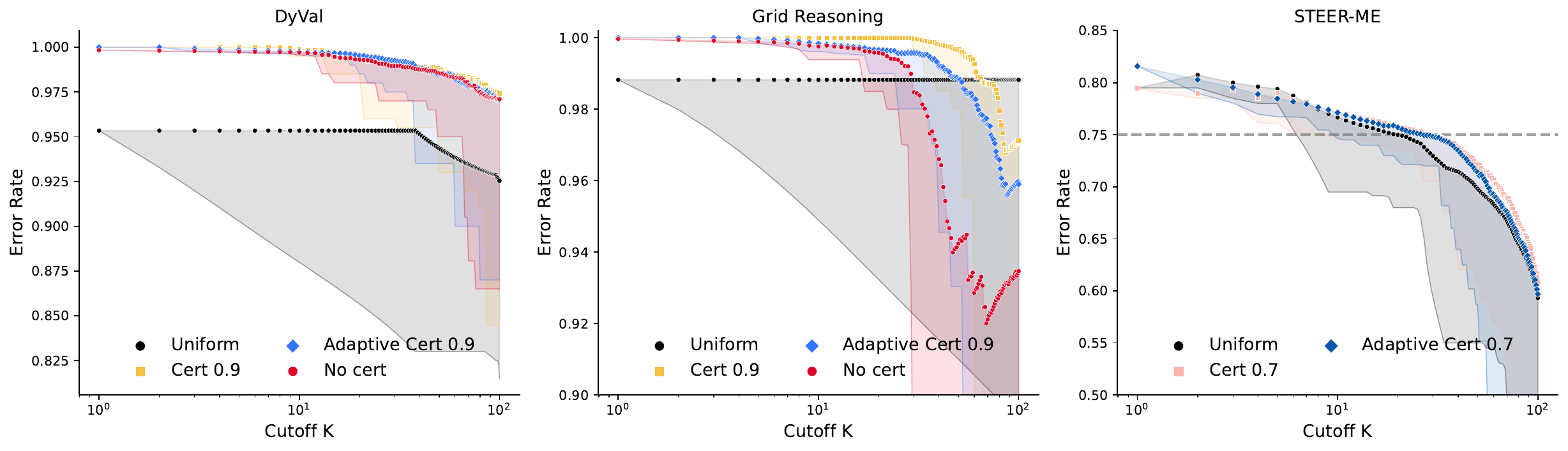}
        \caption{Cumulative average re-evaluation utility as a function of cutoff of top~$K$ configs (log-scaled $x$-axis). COUP variants are ranked by LCB; shaded bands show the running minimum re-evaluated utility. Uniform sampling is ranked by empirical mean; the plot shows the expected cumulative average over all possible ties. The dashed line on STEER-ME marks the error-rate random guessing would achieve.}
        \label{fig:cumavg_reeval}
    \end{figure}
\end{shownto}

Across all three benchmarks, COUP variants often maintained higher cumulative average utility than uniform sampling, particularly at small~$K$. On DyVal and Grid Reasoning, COUP's top-\num{10} configurations achieved near-perfect re-evaluated utility ($>0.99$), whereas the average utility uniform sampling achieved was below \num{0.98} across~$K$. On STEER-ME, the absolute utility levels were lower because its 4-choice MCQA format caps the hardest configurations near error rate $0.75$ rather than $1.0$.

The shaded running-minimum bands in \Cref{fig:cumavg_reeval} reinforce COUP's advantage. For COUP, even the \emph{worst} configuration among the top~$K$ tended to have utility comparable to or exceeding uniform sampling's \emph{average}. This means that COUP's adaptive allocation avoided the false positives of uniform sampling: configurations that appeared hard in an initial small sample but regressed toward moderate difficulty upon re-evaluation.

\paragraph{How Certification Shapes Confidence Bounds} \label{subsec:ci_calibration}

A key advantage of \tool over heuristic search methods is that COUP maintains statistically valid confidence bounds on each configuration's utility. \Cref
{fig:ci_whiskers} shows the empirical value of these bounds by overlaying the original LCB to UCB intervals (shaded bands) with the re-evaluated mean utility (dots) for the top \num{100} configurations ranked by LCB.

\begin{shownto}{colm}
    \begin{figure}[t]
        \centering
        \includegraphics[width=\textwidth]{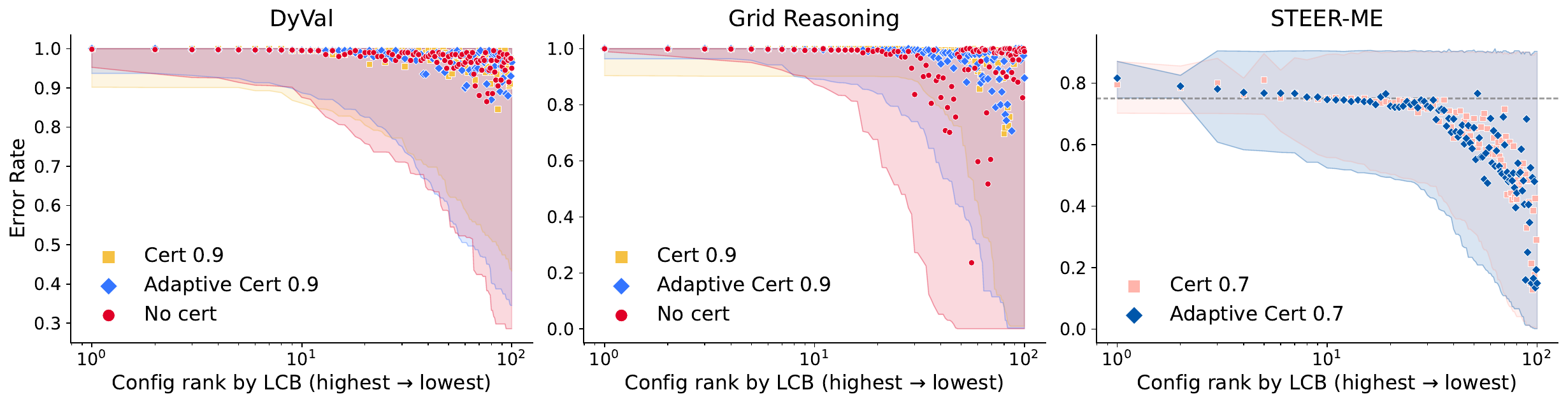}
        \caption{Original LCB--UCB confidence intervals (shaded bands) overlaid with re-evaluated mean utility (dots) for the top \num{100} configurations (rows) across DyVal, Grid Reasoning, and \textsc{STEER-ME} (columns). Log-scaled $x$-axis shows configuration rank by LCB.}
        \label{fig:ci_whiskers}
    \end{figure}
\end{shownto}
On DyVal and Grid Reasoning, the re-evaluated means fell within the original confidence intervals for every configuration in the top~\num{100}, as should be expected from statistically valid bounds. More interestingly, we can examine the width of these intervals, which differed systematically between certified and uncertified runs. Without certification, COUP concentrated samples on its top configurations, producing very tight confidence intervals at the top of the ranking---its estimates were precise because those configurations received the most samples. With certification, top configurations were removed from the active set once their LCB exceeded the threshold, so they received fewer total samples and their intervals were wider. In return, the freed budget was redistributed to tighten confidence intervals on configurations further down the ranking. On Grid Reasoning, for example, certification widened the average CI for the top~\num{10} by \qty{8.1}{\percent} while narrowing CIs for positions 11--100 by \qty{20.0}{\percent}. The practical effect was that certification traded a relatively small amount of precision at the very top---where the configurations were already known to be hard---for more reliable identification of a larger set of hard configurations further down the ranking.

This tradeoff is not binary, however. Adaptive certification raises its effective threshold as more configurations certify, so it does not require the user to know \textit{ex ante} how many configurations will clear the threshold. On DyVal, except at the 1st and 9th configurations, adaptive certification had higher LCBs than both cert~0.9 and no~cert through rank 28; Grid Reasoning, showed a similar trend for $K= 2$--$10$ (\Cref{fig:ci_whiskers}). Fixed-threshold certification certified aggressively, accepting configurations at a lower effective bar than may be desirable and diluted the certified set. Adaptive certification certified fewer configurations but at a higher level of certainty, producing a tighter, more reliable ranking in this regime.

On STEER-ME, the picture was qualitatively similar but with wider intervals and more spread in re-evaluated means, due to the benchmark's MCQA format. The per-sample variance of a Bernoulli outcome is $p(1{-}p)$, which is roughly $0.19$ at $p=0.75$ but only $0.01$ at $p=0.99$. In other words, each evaluation of a STEER-ME template was nearly $20{\times}$ noisier than an evaluation of a near perfectly-hard DyVal or Grid Reasoning template, so COUP required proportionally more samples to achieve the same bound width. 

\paragraph{Utility Functions Surface Different Weak Spots} \label{subsec:utility_comparison}

A distinctive feature of \tool is that users can specify different utility functions to target different notions of model weakness. To illustrate the difference this can make, we compared two utility functions on DyVal: error rate (ER), and complexity-weighted error (CWE).

\Cref{fig:cwe_comparison} takes the top~\num{100} configurations from the ER-optimized and CWE-optimized runs and retroactively scores them under the same metric. The left panel scores both sets by CWE: the CWE-optimized configurations achieve substantially higher complexity-weighted error across all~$K$. The right panel scores both sets by error rate: the ER-optimized configurations dominate instead. Each objective finds configurations that the other misses. Crucially, however, the cost of optimizing CWE instead of ER---measured in raw error rate---was small. At $K=10$ (without certification), CWE-optimized configurations still achieved a \qty{91.1}{\percent} average error rate, compared to \qty{99.4}{\percent} for ER-optimized configurations. Yet the two utilities surfaced qualitatively different templates. Among the top-100 configurations, CWE overwhelmingly selected shallow templates (\qty{89}{\percent} at depth 2--3), whereas ER selected deep templates (\qty{76}{\percent} at depth 7--10), and configuration overlap was negligible (only 6 of 100 configurations appear in both top-\num{100} sets). ER found configurations that were hard in absolute terms---high-complexity templates where the model achieved near-zero accuracy. CWE found configurations that were hard \emph{relative to their complexity}: low-complexity templates where the model failed despite the problem's structural simplicity. Evaluated on CWE's own scale, CWE-optimized configurations scored \qty{4.3}{\times} higher than ER-optimized configurations ($0.46$ vs.\ $0.11$), confirming that CWE surfaced failures in a region of the configuration space that ER left largely unexplored. This gap persisted across $K$: at $K=50$, the CWE-axis ratio was still \qty{3.9}{\times}, while the error-rate gap widened only modestly to $12$ percentage points. 

\begin{figure}[t]
    \centering
    \includegraphics[width=0.9\linewidth]{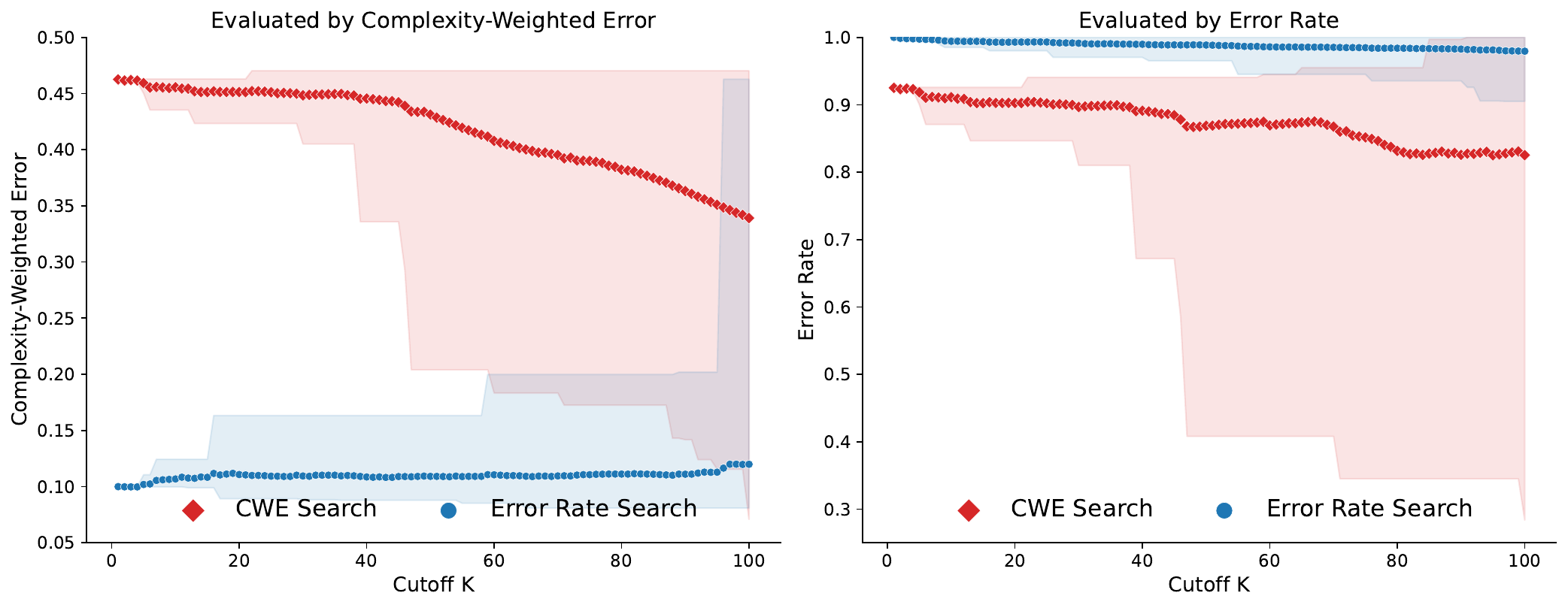}
    \caption{Cumulative average utility on DyVal comparing CWE-optimized (red) and ER-optimized (blue) searches. Plots show runs scored under CWE (left) and ER (right).}
    \label{fig:cwe_comparison}
    \vspace{-0.5em}
\end{figure}

\paragraph{Repulsion Surfaces Diverse Failure Modes} \label{subsec:repulsion}

Finally, we examined the effect of adding a repulsion bonus when certification was enabled at an \qty{80}{\percent} threshold. Repulsion reduced the total number of certified configurations on both DyVal and Grid Reasoning, but it produced a more diverse certified set. On DyVal, the non-repulsive run certified more configurations overall (\num{81} vs.\ \num{46}), but all of them came from the arithmetic task. With repulsion, by contrast, certification still identified hard arithmetic configurations while also surfacing 6 certified configurations from the linear equation task. A similar pattern appeared on Grid Reasoning: repulsion again certified fewer configurations in total (\num{51} vs.\ \num{87}), but the resulting set spanned a broader portion of the benchmark. In \emph{Largest Island}, the certified configurations covered more unique values of the number of islands and island size, and in \emph{Shortest Path} they covered a wider range of blocking probabilities. Repulsion therefore traded off raw certification count for broader coverage of distinct regions of the configuration space, which can be preferable when the goal is not merely to certify many hard configurations, but to surface a wider variety of failure modes.

\section{Discussion} \label{sec:discussion}

Across benchmarks, \tool found more reliably difficult configurations than uniform sampling while reducing false positives from noisy outcomes, and our extensions made this approach practical for modern LLM evaluation pipelines. We intentionally designed \tool as a flexible framework that others can adapt to their own benchmark generators, utility functions, and evaluation goals. At the same time, our methodology is built around objectives that rank configurations by expected utility, that is, by a sample mean or a monotone transformation of one. An important direction for future work is to extend this framework to objectives that instead target a specific performance level. For example, in multiple-choice evaluation one may care most about configurations whose accuracy is close to random guessing, since chance-level performance suggests the model is not extracting usable signal from the task, whereas systematically below-chance performance can itself reflect a learnable pattern. While one could define an aggregate objective based on distance to chance accuracy, that quantity depends on the configuration’s unknown mean accuracy rather than decomposing into per-sample utilities. Handling such objectives would therefore require different confidence-bound and certification tools aimed at identifying configurations near a target performance level, rather than simply ranking them by utility.



\bibliography{ref}
\bibliographystyle{icml2025}

\appendix

\section{Batching Procedure} \label{app:proof}

We now give the full description of our Frozen-LCB simulation batching. 
We implement batching via a two-stage simulation procedure. First, we simulate the
\emph{activation schedule}: starting from the real pre-batch state, we run a
counterfactual sequential version of COUP for $N$ steps in which each sampled
configuration is updated as if its observed outcome were its current empirical
mean. During this simulation, whenever our heuristic for introducing a new
configuration is triggered, we record the corresponding sample index. This gives
a schedule of timesteps at which new configurations would be added to the active
set.

We then construct the actual batch. For an even batch size $N$, we first assign
$\frac{N}{2}$ evaluations to the configuration with the highest current empirical
mean in the real pre-batch state. We use the remaining $\frac{N}{2}$ evaluations
to build the challenger half of the batch iteratively. For
$t=1,\dots,\frac{N}{2}$, we:
(i) add any newly introduced configurations whose recorded activation time is
the current timestep;
(ii) query COUP for the next configuration $i_t$ it would evaluate from the
current temporary state; and
(iii) apply a \emph{temporary simulated update} to $i_t$ by updating its
sufficient statistics (for example, its count $n_{i_t}$ and empirical mean
$\hat\mu_{i_t}$) as if one additional sample with outcome $\tilde y_{i_t}$ had
been observed, and then recomputing its confidence bounds from these updated
statistics.

After the full list of $N$ evaluations has been constructed, all temporary
simulated updates are discarded. We then execute the $N$ real evaluations in
parallel and update COUP using only the realized outcomes.

It remains to specify the simulated outcomes $\tilde y_i$. Because COUP is optimistic, we choose simulated outcomes that keep within-batch UCBs conservative in the sense of (G1). Concretely, we simulate each within-batch evaluation of configuration $i$ as if it returned its \emph{pre-batch} lower confidence bound, and we \emph{freeze} this value throughout batching:
$
\tilde y_i \;=\; \LCB_i(S_0).$
Intuitively, repeatedly updating an arm with outcomes at its (frozen) LCB can only reduce its apparent optimism, encouraging diversity among the remaining batch slots. In the analysis below, we show that this choice yields (G1) and that the configurations selected by the batching loop satisfy (G2).

To establish these guarantees we rely on three regularity conditions on COUP's bounds:
\begin{property}[Monotonicity]
\label{prop:coup_monotone}
For all states $S$ and configurations $i$, the map $m\mapsto \EUCB_{S,i}(m)$ is nondecreasing.
\end{property}

\begin{property}[Simultaneous LCB validity]
\label{prop:coup_lcb_valid}
\(
\Pr\Big(\forall i\in\mathcal{I}:~ L_i \le \mu_i\Big) \ge 1-\delta.
\)
\end{property}

\begin{property}[Arm-separability and pessimistic updates are nonincreasing]
\label{prop:local_updates}
Updating $j\neq i$ does not change $\UCB_i$. Moreover, for any state $S$ and $i\in\mathcal I$,
\(
y\le \hat\mu_i(S)\quad\Longrightarrow\quad \UCB_i(\Update(S,i,y)) \le \UCB_i(S).
\)
\end{property}

\firsttheorem*
\begin{proof}
Let $E \coloneqq \{\forall i\in\mathcal{I}:\ L_i \le \mu_i\}$. By \Cref{prop:coup_lcb_valid}, $\Pr(E)\ge 1-\delta$.
On $E$, for every step $t$ we have $L_{i_t}\le \mu_{i_t}$. By \Cref{prop:coup_monotone}, $m\mapsto \EUCB_{S_t,i_t}(m)$ is nondecreasing for every (possibly simulated-constructed) state $S_t$, hence
$\EUCB_{S_t,i_t}(L_{i_t}) \le \EUCB_{S_t,i_t}(\mu_{i_t})$ for each $t$ on the same event $E$.
Therefore the inequalities hold simultaneously for all $t=1,\dots,K$ with probability at least $1-\delta$.
\end{proof}

\begin{remark}[Least pessimistic frozen simulated mean]
Under Assumption~\ref{prop:coup_monotone}, among all deterministic frozen choices $\tilde{m_i}(D)$ satisfying
$\Pr(\forall i\in\mathcal{I}: \widetilde{m}_i(D)\leq \mu_i)\geq 1-\delta$, the choice $\tilde{m_i}(D)=L_i(D)$ is pointwise maximal on the event of simultaneous LCB validity. In this sense, using the LCB is the \emph{least pessimistic} frozen simulated mean that preserves a $(1-\delta)$ conservativeness guarantee.
\end{remark}

We now show that the set of configurations sampled by pessimistic sampling is a superset of those that would be sampled in an ``expected outcome'' world (i.e. where each sample is simulated as the true mean).

\secondtheorem*
\begin{proof}
We only need to argue this fact for the set of UCB challengers. Consider the sequential algorithm as a process of lowering a UCB threshold $v$. For any configuration $i$, let $n_i(v)$ be the number of samples required to reduce $\UCB_i$ to or below $v$. The algorithm halts at a specific threshold $v^*$ such that the total samples $\sum_i n_i(v^*) = K$.

First, observe that for any fixed threshold $v$, the simulation requires fewer or equal samples to drive a configuration's UCB below $v$ compared to the expected real trajectory. This follows from Theorem~\ref{thm:coup_fantasy}: since $L_i \le \mu_i$, the simulated UCB decreases no slower than the expected real UCB. Thus,
$
n_i^{\mathrm{sim}}(v) \le n_i^{\mathrm{real}}(v)
\qquad\text{for all } i,v.
$

Now consider the effect of newly introduced configurations. By assumption, whenever a new configuration appears, it appears at the same timestep in both worlds and with the same maximal UCB. Hence, at the moment of introduction, it joins the set of arms above the current threshold (indeed, with maximal possible UCB), and must be sampled until its UCB is reduced to that threshold. The same comparison as above applies to such an arm from the moment it is introduced onward: for any fixed threshold $v$, the pessimistic simulation needs no more samples than the expected-outcome process to reduce that arm's UCB to $v$. Therefore, the introduction of new arms does not alter the threshold argument; it only enlarges the set of arms that must be driven down to the current threshold, and the pessimistic simulation can do this at least as quickly as the expected-outcome world.

Now, let $v^*_{\text{real}}$ be the final stopping threshold of the true process. The total sample cost for the simulation to reach this same threshold is:
$$
\sum_i n_i^{\mathrm{sim}}(v^*_{\text{real}})
\le
\sum_i n_i^{\mathrm{real}}(v^*_{\text{real}})
=
K.
$$
Since the simulation achieves the real quality standard $v^*_{\text{real}}$ using at most $K$ samples, it must lower its threshold at least as far to consume the full budget. Thus, the simulation's effective stopping threshold satisfies
$$
v^*_{\text{sim}} \le v^*_{\text{real}}.
$$

A configuration is included in the set $\mathcal{S}$ if and only if its UCB at the time it becomes active exceeds the eventual stopping threshold. Since $v^*_{\text{sim}} \le v^*_{\text{real}}$, any configuration $i \in \mathcal{S}_{\text{real}}$ necessarily also satisfies this condition for the simulation, and is therefore included in $\mathcal{S}_{\text{sim}}$.

The remainder of the argument follows from the simple fact that $K$ is constant in both worlds: if the set of configurations that the $K$ samples are spread across is larger in the simulation, then the average number of samples per configuration in $\mathcal{S}_{\text{real}}$ must be strictly smaller.
\end{proof}

\section{Experimental Setup} \label{app:experimental_setup}

\subsection{COUP Hyperparameters}

All experiments use the same COUP specifications unless otherwise noted. Each batch consists of $B = 20$ tickets (model calls). The batch is split into two phases: half for the incumbent (highest empirical mean) and half allocated by UCB. Within the UCB phase, COUP may add new configurations to the active set when the squared epsilon-prime condition is met (see \Cref{app:proof}). The total budget per run is \num{20000} model calls, yielding \num{1000} batches. COUP is initialized with $n_0 = 50$ random configurations. The exploration parameter is $e = 1.0$, confidence parameter $\delta = 0.01$, and random exploration probability $1/2$ (i.e., one-half of UCB-phase selections are replaced by random configuration proposals when the surrogate model is active).

\subsection{Scalability Modifications} \label{app:scalability}

COUP's theoretical guarantees apply to the exact algorithm, but we made two practical modifications:

\begin{enumerate}[leftmargin=15pt]
    \item \textbf{Deferred bounds refresh on pool growth.} Adding a new configuration to the active set changes the per-arm confidence level $\delta_i = \delta / (26.71 \cdot n^2 \cdot m_i^2)$ for every existing configuration, technically requiring a full bounds recomputation. We defer this recomputation until the pool size has grown by a factor of \qty{1.5}{\times} since the last refresh, amortizing the $O(n)$ update cost. Configurations with fewer than $2$ samples are skipped during the refresh, as their bounds are maximally wide regardless.
    \item \textbf{Surrogate model.} A random-forest surrogate is trained on observed (configuration, utility) pairs and used to propose candidate configurations during exploration. The surrogate is retrained whenever the number of observed configurations with at least $2$ samples exceeds the count at the last training step.
\end{enumerate}

These modifications do not affect the validity of the confidence bounds for configurations that have already been refreshed; they only delay the point at which new-pool-size--adjusted bounds become available, introducing a conservative lag of at most one refresh cycle.

\subsection{Prompting and Response Parsing}

All models are prompted to place their final answer inside a \verb|\boxed{}| delimiter. For numeric tasks (DyVal, Grid Reasoning), the prompt reads: \emph{``Please reason step-by-step \ldots\ Provide your final answer inside \textbackslash boxed\{\,\}. For example, respond with \textbackslash boxed\{42\}.''} For MCQA tasks (STEER-ME), the prompt asks for a letter inside \verb|\boxed{}|. Responses are parsed by extracting the last \verb|\boxed{...}| token with nested-brace support; if no \verb|\boxed{}| is found, fallback heuristics attempt ``\texttt{Answer: X}'' and bare ``\texttt{boxed X}'' patterns.

Timeouts, generation failures, and unparseable responses are all treated as correct (utility $= 0.0$ under error-rate), so that COUP does not chase these kinds of errors.

\subsection{Model and Inference Settings}

\Cref{tab:model_settings} summarizes the inference settings for each model. OpenAI reasoning-series models (\texttt{gpt-5-mini}, \texttt{gpt-5.4-nano}) use API-level reasoning effort controls. 

\begin{table}[h]
    \centering
    \small
    \begin{tabular}{lllc}
    \toprule
    Model & Provider & Reasoning effort & Temperature \\
    \midrule
    \texttt{gpt-5.4-nano-2026-03-17} & OpenAI & none & (API default) \\
    \texttt{gpt-5-mini-2025-08-07} & OpenAI & minimal & (API default) \\
    \bottomrule
    \end{tabular}
    \caption{Model inference settings.}
    \label{tab:model_settings}
\end{table}

Per-request timeouts are set to \num{120} seconds (2 minutes). Requests that exceed this limit are recorded as timeouts and treated as correct. \Cref{tab:error_rates} reports the empirical rates of timeouts and unparseable responses (pooled across all runs on each benchmark).

\begin{table}[h]
    \centering
    \small
    \begin{tabular}{llrr}
    \toprule
    Model & Benchmark & Timeout rate & Unparse rate \\
    \midrule
    \multirow{3}{*}{\texttt{gpt-5.4-nano}} & DyVal & \qty{<0.01}{\percent} & \qty{<0.01}{\percent} \\
     & Grid Reasoning & \qty{0.35}{\percent} & \qty{5.2}{\percent} \\
     & STEER-ME & \qty{0}{\percent} & \qty{1.2}{\percent} \\
    \midrule
    \multirow{3}{*}{\texttt{gpt-5-mini}} & DyVal & \qty{0}{\percent} & \qty{4.6}{\percent} \\
     & Grid Reasoning & \qty{<0.01}{\percent} & \qty{9.3}{\percent} \\
     & STEER-ME & \qty{0}{\percent} & \qty{7.8}{\percent} \\
    \bottomrule
    \end{tabular}
    \caption{Timeout and unparseable-response rates, pooled across all runs (COUP variants + uniform) on each benchmark. Unparseable responses are those where the model produced output but no valid answer could be extracted from the \texttt{\textbackslash boxed\{\}} delimiter or fallback heuristics. Both are treated as correct.}
    \label{tab:error_rates}
\end{table}

Timeouts are negligible for both models. Unparseable responses are concentrated on Grid Reasoning, where the model must produce a path description or numeric area---\texttt{gpt-5-mini}'s higher unparse rate (\qty{9}{\percent} vs.\ \qty{5}{\percent}) reflects its tendency to produce verbose reasoning that omits the \verb|\boxed{}| delimiter. On STEER-ME, \texttt{gpt-5-mini} also has a higher unparse rate (\qty{8}{\percent} vs.\ \qty{1}{\percent}), likely because its ``minimal'' reasoning effort occasionally produces malformed MCQA answers. Because all such failures are counted as correct, they do not inflate measured error rates and thus the utility signal.

\section{Benchmarks} \label{app:benchmarks}

\textsc{DyVal} \citep{zhu2024dyvaldynamicevaluationlarge} is a benchmark in which each template corresponds to a structured reasoning object with associated features (e.g., a notion of complexity). Instances are generated by sampling concrete realizations from the object while preserving its underlying structure. In our experiments, the template identifier selects the underlying reasoning object (and any discrete template attributes exposed by the benchmark), while the generator samples i.i.d.\ realizations used to estimate accuracy and to support utilities that depend on template-level metadata such as complexity. The configuration space comprises six parameters: dataset type (5 reasoning tasks), depth (2--10), children per node (2--4), extra DAG links (0--3), random description perturbations (0--3), and node ordering (3 schemes). After forbidding combinations that yield computationally prohibitive tree sizes (e.g., both deep and wide structures), the search space contains $|\mathcal{X}| = 4{,}440$ valid templates.

\begin{shownto}{icml}
    \begin{figure*}[h]
        \centering
        \includegraphics[width=\textwidth]{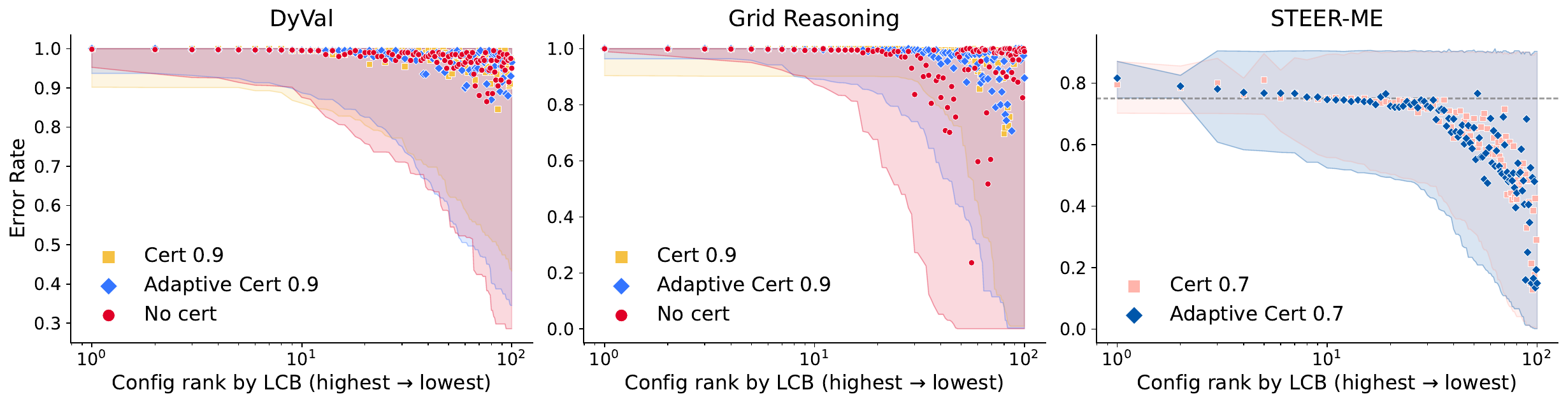}
        \vspace{-1.5em}
        \caption{Original LCB--UCB confidence intervals (shaded bands) overlaid with re-evaluated mean utility (dots) for the top \num{100} configurations (rows) across DyVal, Grid Reasoning, and \textsc{STEER-ME} (columns). Log-scaled $x$-axis shows configuration rank by LCB. Re-evaluated means consistently fall within the original intervals, validating COUP's confidence bounds under the batching extensions of \Cref{sec:coup}.}
        \label{fig:ci_whiskers}
    \end{figure*}
\end{shownto}

\textsc{STEER-ME} \citep{raman2025steermeassessingmicroeconomicreasoning} is a benchmark for non-strategic microeconomic reasoning. It organizes questions around $58$ economic \emph{elements} (e.g., optimization, demand, comparative statics) and supports controlled variation in $10$ \emph{domains} (e.g., medical, finance, public policy) and $15$ types (e.g., functional forms, bits of precision, number of optimization dimensions). A template identifier in \textsc{STEER-ME} is a structured tuple that selects an element, a domain, and a type specification; sampling an instance then draws fresh problem parameters (e.g., numerical values and text that does not effect the logic of the question like the gender or first/second person) consistent with that template. We used the public MCQA version of the benchmark whose search space contains $303$ templates.


\textsc{Grid Reasoning} \citep{stojanovski2025reasoninggymreasoningenvironments} is a dynamic, programmatic benchmark we built by combining two \texttt{reasoning-gym} task families: \emph{Shortest Path} and \emph{Largest Island}. Both tasks generate 2D grid instances from a distribution over grids (e.g., grid size and randomly generated layouts), and each instance asks for a deterministic property of the sampled grid (e.g., a shortest-path quantity for \emph{Shortest Path}, or the size of the largest connected component for \emph{Largest Island}). We define template identifiers that expose the main discrete structure of the generator (task family, and optionally coarse bins for generation settings such as grid size or density), while the instance generator fills in the remaining degrees of freedom by sampling a fresh grid realization. This yields a large space of potential instances even for a fixed template identifier, and a large overall evaluation space once we vary task family and generation settings. The configuration space is mixed discrete--continuous: shared grid dimensions ($21 \times 21$ choices for rows and columns in $[5,25]$), task-specific parameters (a continuous blocking probability for \emph{Shortest Path}; island count and size ranges for \emph{Largest Island} subject to min $\leq$ max constraints). Overall it contains $|\mathcal{X}| > 2.3 \times 10^6$ discrete templates plus one continuous degree of freedom.

\section{Identified Weak Spots} \label{app:weak_spots}

Tables~\ref{tab:top5_dyval}--\ref{tab:top5_steerme} list the five highest-LCB configurations found by each COUP variant on each benchmark (gpt-5.4-nano, error-rate utility).

\begin{table}[h]
\centering
\small
\begin{tabular}{clccccccc}
\toprule
Variant & Rank & Type & Depth & Children & Extra links & Rand.\ desc. & Order & LCB \\
\midrule
\multirow{5}{*}{No cert}
 & 1 & arith. & 10 & 2 & 0 & 2 & rev. & 0.952 \\
 & 2 & arith. & 10 & 2 & 0 & 1 & rev. & 0.936 \\
 & 3 & arith. & 9  & 2 & 0 & 0 & rand. & 0.927 \\
 & 4 & arith. & 9  & 2 & 0 & 0 & rev. & 0.925 \\
 & 5 & arith. & 10 & 2 & 1 & 2 & topo. & 0.914 \\
\midrule
\multirow{5}{*}{Cert 0.9}
 & 1 & arith. & 10 & 2 & 0 & 3 & topo. & 0.902 \\
 & 2 & arith. & 9  & 2 & 1 & 2 & rand. & 0.901 \\
 & 3 & arith. & 10 & 2 & 0 & 3 & rev. & 0.901 \\
 & 4 & arith. & 10 & 2 & 0 & 0 & rev. & 0.901 \\
 & 5 & arith. & 9  & 2 & 0 & 0 & rand. & 0.901 \\
\midrule
\multirow{5}{*}{Adaptive}
 & 1 & arith. & 9  & 2 & 1 & 2 & rev. & 0.937 \\
 & 2 & arith. & 10 & 2 & 0 & 0 & rev. & 0.937 \\
 & 3 & arith. & 10 & 2 & 0 & 1 & rand. & 0.932 \\
 & 4 & arith. & 10 & 2 & 0 & 3 & rev. & 0.926 \\
 & 5 & arith. & 10 & 2 & 0 & 0 & rand. & 0.924 \\
\bottomrule
\end{tabular}
\caption{Top-5 configurations on \textbf{DyVal} by LCB. Dataset types: 1=arithmetic, 2=linear equation, 3=Boolean logic, 4=deductive logic, 5=abductive logic. Order: 1=topological, 2=reversed, 3=random.}
\label{tab:top5_dyval}
\end{table}

\begin{table}[h]
\centering
\small
\begin{tabular}{clccccc}
\toprule
Variant & Rank & Task & Rows & Cols & Islands / $p_{\text{blocked}}$ & LCB \\
\midrule
\multirow{5}{*}{No cert}
 & 1 & largest island & 21 & 9 & 7 isl., size 2--20 & 0.990 \\
 & 2 & shortest path  & 21 & 23 & --- & 0.974 \\
 & 3 & largest island & 14 & 13 & 8 isl., size 2--15 & 0.963 \\
 & 4 & largest island & 8  & 19 & 5 isl., size 4--16 & 0.956 \\
 & 5 & shortest path  & 18 & 22 & --- & 0.950 \\
\midrule
\multirow{5}{*}{Cert 0.9}
 & 1 & largest island & 19 & 10 & 9 isl., size 5--17 & 0.904 \\
 & 2 & shortest path  & 21 & 14 & --- & 0.903 \\
 & 3 & largest island & 21 & 16 & 8 isl., size 4--15 & 0.902 \\
 & 4 & largest island & 15 & 12 & 8 isl., size 0--18 & 0.902 \\
 & 5 & shortest path  & 22 & 12 & --- & 0.902 \\
\midrule
\multirow{5}{*}{Adaptive}
 & 1 & largest island & 21 & 9  & 7 isl., size 2--20 & 0.964 \\
 & 2 & shortest path  & 21 & 23 & --- & 0.964 \\
 & 3 & largest island & 18 & 10 & 9 isl., size 5--10 & 0.964 \\
 & 4 & largest island & 7  & 14 & 9 isl., size 0--16 & 0.964 \\
 & 5 & largest island & 17 & 14 & 6 isl., size 1--19 & 0.959 \\
\bottomrule
\end{tabular}
\caption{Top-5 configurations on \textbf{Grid Reasoning} by LCB.}
\label{tab:top5_grid}
\end{table}

\begin{table}[h]
\centering
\small
\begin{tabular}{clllcc}
\toprule
Variant & Rank & Element & Domain & Type & LCB \\
\midrule
\multirow{5}{*}{Cert 0.7}
 & 1 & intertemporal cons.\ smooth. & education & numerical & 0.702 \\
 & 2 & capital market distortions & consumer goods & output & 0.701 \\
 & 3 & intertemporal cons.\ smooth. & entertainment & numerical & 0.701 \\
 & 4 & deriving demand & education & good\_2 & 0.700 \\
 & 5 & intertemporal cons.\ smooth. & medical & numerical & 0.698 \\
\midrule
\multirow{5}{*}{Adaptive}
 & 1 & intertemporal cons.\ smooth. & medical & numerical & 0.751 \\
 & 2 & intertemporal cons.\ smooth. & entertainment & numerical & 0.751 \\
 & 3 & deriving demand & education & good\_2 & 0.608 \\
 & 4 & intertemporal cons.\ smooth. & envir.\ policy & numerical & 0.583 \\
 & 5 & deriving labor supply & finance & Cobb--Douglas & 0.579 \\
\bottomrule
\end{tabular}
\caption{Top-5 configurations on \textbf{STEER-ME} by LCB. Element names are abbreviated.}
\label{tab:top5_steerme}
\end{table}

All three benchmarks show clear thematic patterns. On DyVal, every top-5 configuration across all variants was an arithmetic task at maximum or near-maximum depth (9--10) with narrow trees (2 children per node)---deep, narrow arithmetic reasoning chains were consistently the hardest for the model. On Grid Reasoning, the top configurations were split between \emph{largest island} (with many islands on medium-to-large grids) and \emph{shortest path} (on large grids), with no~cert and adaptive~cert finding the same rank-1 and rank-2 configurations. On STEER-ME, \emph{intertemporal consumption smoothing} with numerical question types dominated across both variants, appearing in 3 of 5 slots for cert~0.7 and 3 of 5 for adaptive~cert.

\section{GPT-5 (Mini) Results} \label{app:cross_model}

The main body presents results for \texttt{gpt-5.4-nano}. Here we show that the qualitative trends held for \texttt{gpt-5-mini}. \Cref{fig:cumavg_mini,fig:ci_mini} replicate the cumulative average and CI whisker plots from the main body for \texttt{gpt-5-mini}. 

\paragraph{Implementation note.} The \texttt{gpt-5-mini} runs used a \emph{pessimistic} simulation for the add-configuration check: within each batch, COUP simulates future observations at the arm's current LCB (a worst-case assumption) when deciding whether to add a new configuration to the active set. The current implementation instead uses the \emph{expected-mean} simulation, which assumes future observations equal the arm's empirical mean. Because the pessimistic simulation underestimates UCBs, the epsilon-prime condition triggers later, so new configurations enter the active set $1$--$2$ batches behind where the expected-mean schedule would place them. The confidence bounds and final rankings are unaffected---only the timing of configuration additions is shifted. See \Cref{app:proof} for the batching procedure.

\paragraph{COUP's advantage is larger on gpt-5-mini.} On Grid Reasoning, all COUP variants achieve a top-\num{10} cumulative average re-evaluated utility $\geq 0.99$ (cert~0.9: $0.991$, no~cert: $0.990$), while uniform sampling reaches only $0.951$---a gap of ${\sim}4$ percentage points, wider than the corresponding gap on \texttt{gpt-5.4-nano}. On DyVal, adaptive certification again leads ($1.0$ through $K=10$); no~cert and cert~0.9 are slightly lower ($0.989$ and $0.984$) but still well above uniform ($0.971$). Even at $K=50$, every COUP variant maintains a cumulative average above $0.978$ on both benchmarks, while uniform drops to $0.965$ on DyVal and $0.930$ on Grid Reasoning. On STEER-ME, all methods cluster near $0.88$--$0.90$ at $K=10$, consistent with the noise characteristics of 4-choice MCQA.

The CI whisker plots (\Cref{fig:ci_mini}) confirm valid Hoeffding coverage: re-evaluated means fall within the original confidence intervals throughout the top~\num{100}. As with \texttt{gpt-5.4-nano}, no~cert produces the tightest intervals at the top of the ranking while cert~0.9 produces tighter intervals further down, reflecting the same budget-reallocation tradeoff.

\paragraph{CWE and ER find different weak spots.} \Cref{fig:cwe_mini} replicates the CWE-vs-ER comparison on DyVal for \texttt{gpt-5-mini}. The same asymmetry holds: when scored by CWE, the CWE-optimized configurations dominate; when scored by error rate, the ER-optimized configurations dominate. The structural separation is nearly identical to \texttt{gpt-5.4-nano}: CWE overwhelmingly selects shallow templates (\qty{89}{\percent} at depth~2--3), ER selects deep templates (\qty{75}{\percent} at depth~7--10), and only 8 of the 95 re-evaluated configurations appear in both top sets.

\begin{figure}[h]
    \centering
    \includegraphics[width=\textwidth]{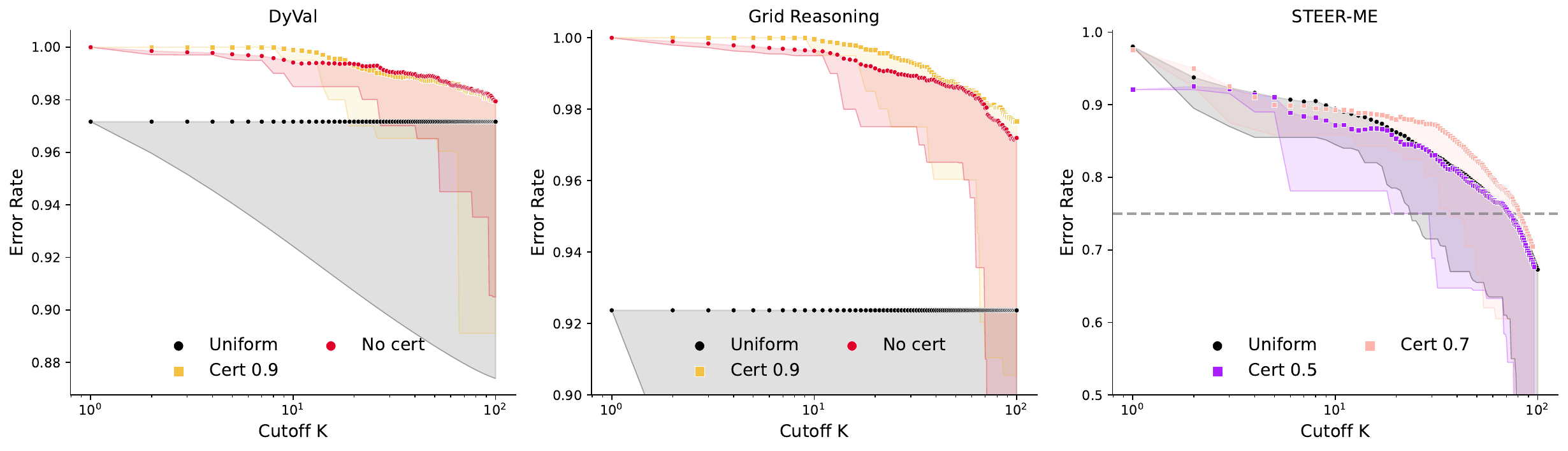}
    \caption{Cumulative average re-evaluation utility for gpt-5-mini. Same format as \Cref{fig:cumavg_reeval}. COUP variants consistently outperformed uniform sampling, and the running-minimum bands confirm that even the worst COUP-selected configurations tended to exceed uniform's average.}
    \label{fig:cumavg_mini}
\end{figure}

\begin{figure}[h]
    \centering
    \includegraphics[width=\textwidth]{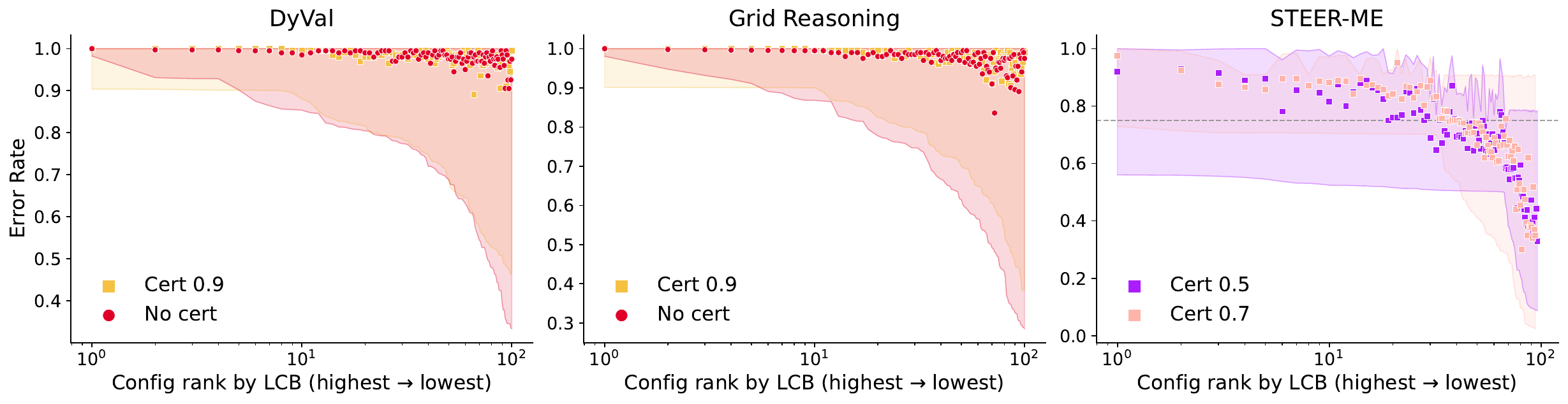}
    \caption{Confidence interval calibration for gpt-5-mini. Same format as \Cref{fig:ci_whiskers}. Re-evaluated means fell within the original intervals, and the certification-dependent width patterns matched those observed for gpt-5.4-nano.}
    \label{fig:ci_mini}
\end{figure}

\begin{figure}[h]
    \centering
    \includegraphics[width=0.9\textwidth]{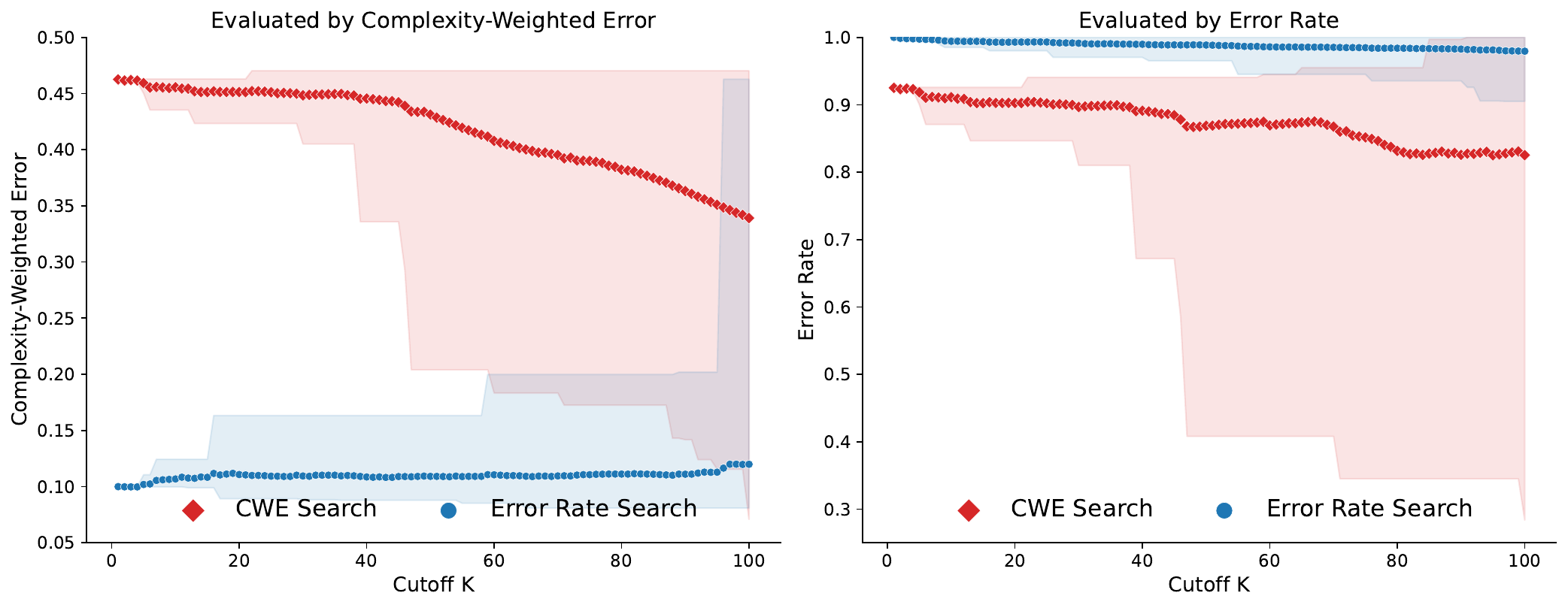}
    \caption{CWE vs.\ error-rate comparison on DyVal for gpt-5-mini. Same format as \Cref{fig:cwe_comparison}. The asymmetry between objectives replicates: each finds configurations the other misses.}
    \label{fig:cwe_mini}
\end{figure}
\end{document}